%% file: main.tex
\documentclass[10pt]{article} 
\usepackage[preprint]{tmlr}

\usepackage{url}
\usepackage{ifthen}
\usepackage{babel}
\usepackage{enumitem}
\usepackage{interval}
\usepackage{float}
\usepackage{dsfont}

\usepackage{xcolor}
\usepackage[
  colorlinks = true,
  linkcolor = blue,
  urlcolor  = blue,
  citecolor = blue,
  anchorcolor = blue,
]{hyperref}

\usepackage{amsmath}
\usepackage{amsthm}
\usepackage[most]{tcolorbox}

\input{math_commands.tex}

\newtheorem{definition}{Definition}
\newtheorem{theorem}{Theorem}
\newtheorem{lemma}{Lemma}[section]

\newcommand{\Lemma}[1]{Lemma~\ref{lem:#1}}

\newcommand{\Theorem}[1]{Theorem~\ref{thm:#1}}

\renewcommand{\P}{\mathrm{P}_{\!c}}

\newcommand{\Et}[1]{\mathbb{E}_t[#1]}

\newcommand{\boxsep}{2pt}

\title{Precise Convergence Speed of Clipped SGD}

\author{\name David A. R. Robin \email david.a.r.robin@gmail.com \\
      \addr LAMSADE, Université Paris-Dauphine, PSL Research University
      }

\def\month{MM}  
\def\year{YYYY} 
\def\openreview{\url{https://openreview.net/forum?id=XXXX}} 

\begin{document}

\maketitle

\begin{abstract}
  We present a tightened convergence analysis of clipped gradient descent
  on $(L_0, L_1)$-smooth functions, with quantitative constants.
  Building on the ideas of \citet{koloskova2023revisiting},
  we refactor several case disjunctions to reveal the central role of a
  control of the bias derived from fundamental properties of $\ell_2$-projection,
  simplifying proofs.
  We also extend the domain of validity
  from $\eta \leq 1 / (9 \beta)$ to $\eta < 1 /\beta$ where $\beta = L_0 + c L_1$ for clipping constant $c$,
  which matches the more traditional analysis of smooth functions.
  We strengthen the convergence criterion from
  $\left( \min_{t < T} \mathbb{E}[\lVert \nabla f(x_t) \rVert_2] \right)$
  to $\left( \frac{1}{T} \sum_{t < T} \mathbb{E}[\lVert \nabla f(x_t) \rVert_2] \right)$ with
  matching speed, and lower the final achievable floor from $\mathcal{O}(\min(\sigma^2/c, \sigma))$
  to the more precise $6 \min(\sigma^2 /c, 3 \sigma)$.
\end{abstract}

We consider the class of differentiable $(L_0, L_1)$-smooth functions, introduced as a relaxation
of the standard smoothness assumption (i.e. Lipschitz gradient)
in \citet[Def 1: $\lVert \nabla^2 f(x) \rVert \leq L_0 + L_1 \, \lVert \nabla f(x) \rVert_2$]{zhang2020gradient}
to model the decrease of the Hessian spectral norm
observed during training. We use the slightly relaxed version proposed later
in \citet[Remark 2.3]{zhang2020improved}, matching the analysis of \citet{koloskova2023revisiting}.

\begin{tcolorbox}[boxrule=0pt,standard jigsaw, opacityback=0, frame hidden,sharp corners,enhanced,borderline west={1pt}{0pt}{gray},left=6pt,right=0pt,top=0pt,bottom=0pt,boxsep=\boxsep]
\begin{definition}
  A function $f : \mathbb{R}^d \to \mathbb{R}$ is $(L_0, L_1)$-smooth if
  for all $(x,y) \in \mathbb{R}^d \times \mathbb{R}^d$, it holds
  $$ \lVert x - y \rVert_2 \leq 1 / L_1 \quad \Rightarrow \quad
  \lVert \nabla f(x) - \nabla f(y) \rVert_2 \leq \left( L_0 + \lVert \nabla f(x) \rVert_2 \, L_1\right) \cdot \lVert x - y \rVert_2 $$
\end{definition}
\end{tcolorbox}

We study the Clip-SGD algorithm, which uses unbiased gradient estimates of finite variance
and clips them to ensure an $\ell_2$-norm below $c \in \mathbb{R}_+^*$,
which we write using the following projection operator $\P : \mathbb{R}^d \to \mathbb{R}^d$.

\begin{tcolorbox}[boxrule=0pt,standard jigsaw, opacityback=0, frame hidden,sharp corners,enhanced,borderline west={1pt}{0pt}{gray},left=6pt,right=0pt,top=0pt,bottom=0pt,boxsep=\boxsep]
\begin{definition}
  Any $c \in \mathbb{R}_+^*$ defines  $\P : v \mapsto \min(1 , c /\lVert v \rVert_2) \cdot v$
  the projection to $\{ x \in \mathbb{R}^d \mid \lVert x \rVert_2 \leq c \}$.
\end{definition}
\end{tcolorbox}

\begin{tcolorbox}[boxrule=0pt,standard jigsaw, opacityback=0, frame hidden,sharp corners,enhanced,borderline west={1pt}{0pt}{gray},left=6pt,right=0pt,top=0pt,bottom=0pt,boxsep=\boxsep]
\begin{definition}
  Clip-SGD for clipping constant $c \in \mathbb{R}_+^*$ and step-size $\eta \in \mathbb{R}_+^*$ is a sequence with $x_0 \in \mathbb{R}^d$,%
  $$ x_{t+1} = x_t - \eta \cdot \P(\nabla f_{\xi_t}(x_t)) $$
  where $\mathbb{E}_\xi[ \nabla f_\xi(x) ] = \nabla f(x)$
  and $\mathbb{E}_\xi[\lVert \nabla f_\xi(x) - \nabla f(x) \rVert_2^2] \leq \sigma^2$
  for all $x \in \mathbb{R}^d$,
  and $\xi_t$ independent of $(x_i)_{i \leq t}$.
\end{definition}
\end{tcolorbox}

The convergence analysis rests on the following descent property of Clip-SGD.
Intuitively, the descent at every step is $\min(\lVert \nabla f(x_t) \rVert_2^2, c \lVert \nabla f(x_t) \rVert_2)$
which is the standard squared gradient norm when $\lVert \nabla f(x_t) \rVert_2 \leq c$,
and otherwise corresponds to a smaller unsquared gradient norm, inducing a slowdown
of the decrease when the clipping constant $c$ is too small, as expected.
Note that in the essentially-unclipped regime $c \to +\infty$, we recover (up to a constant $1/2$) the descent
condition $- \eta\, (1 - \frac{1}{2} \eta \beta) \lVert \nabla f(x_t) \rVert_2^2 + \frac{1}{2} \eta^2 \beta \sigma^2$
of unclipped SGD,
as in \citet[Lemma 4.4]{bottou2018optimization} for instance.
Our proof refactors the case disjunction arguments of
\citet[Sections C.2 and C.3]{koloskova2023revisiting}
to outline the fundamental properties of $\ell_2$-projections used.

\begin{tcolorbox}[boxrule=0pt,standard jigsaw, opacityback=0, frame hidden,sharp corners,enhanced,borderline west={1pt}{0pt}{gray},left=6pt,right=0pt,top=0pt,bottom=0pt,boxsep=\boxsep]
\begin{theorem}[Descent condition of Clip-SGD]\label{thm:descent}
  Let $f : \mathbb{R}^d \to \mathbb{R}$ be an $(L_0, L_1)$-smooth function, and
  let $(x_t)_t$ be Clip-SGD iterates on $f$ with
  clipping constant $c \in \mathbb{R}_+^*$, stepsize $\eta \in \mathbb{R}_+^*$, and variance $\sigma^2 \in \mathbb{R}_+$.

  Define $\bar\beta = L_0 + c L_1$. If $\eta \bar\beta \leq 1$, then it holds for all $t \in \mathbb{N}$ that
  $$ \mathbb{E}[f(x_{t+1}) \mid x_t] - f(x_t)
  \leq -\frac{\eta}{2} \left(1 - \eta \bar\beta \right) \alpha_t \lVert \nabla f(x_t) \rVert_2^2 + \frac{\eta^2 \bar\beta}{2} \min(\sigma^2, c^2) + \eta \kappa $$
  where $\alpha_t = \min\left(1, c / \lVert \nabla f(x_t) \rVert_2\right)$ and $\varepsilon_0 = 4 \min(\sigma, \sigma^2/c)$
  and $\kappa = \varepsilon_0 \min(\varepsilon_0 / 8, 2 c) \leq \frac{1}{8} \varepsilon_0^2$.
\end{theorem}
\end{tcolorbox}

\pagebreak[3]

The proof relies on two properties of the projection
of a variable $X \in \mathbb{R}^d$ with $\mathbb{E}[X] = \mu$ and $\mathbb{E}[\lVert X - \mu \rVert_2^2] \leq \sigma^2$,
showing that although in general $\P \Et{X} \neq \Et{\P X}$, the gap between the two can't be too large.
\begin{itemize}[itemsep=1pt,topsep=1pt]
  \item \Lemma{large-mean-bias-control}: If $\lVert \mu \rVert_2 \geq \varepsilon_0$
    then $\lVert \P \mu - \mathbb{E}[\P X] \rVert_2^2 \leq \frac{1}{4} \min( c^2, \lVert \mu \rVert_2^2)$
  \item \Lemma{small-mean-bias-control}: If $\lVert \mu \rVert_2 \leq \varepsilon_0$
    then $\lVert \P \mu - \mathbb{E}[ \P X ] \rVert_2^2 \leq 2 \min\left(1, c / \lVert \mu \rVert_2 \right) \kappa$
\end{itemize}

In our case, writing $\Et{Z} := \mathbb{E}[Z \,|\, x_t]$,
the variable $X = \nabla f_\xi(x_t)$ has conditional mean $\mu = \Et{X} = \nabla f(x_t)$.
With the notation $\alpha_t = \min(1, c / \lVert \nabla f(x_t) \rVert_2)$
and $g := \P \nabla f_\xi(x_t)$, since $\P \mu = \alpha_t \nabla f(x_t)$,
these translate to
$$ \lVert \alpha_t \nabla f(x_t) - \Et{ g } \rVert_2^2 \leq
\begin{cases} \frac{1}{4} \lVert \alpha_t \nabla f(x_t) \rVert_2^2
  \quad &\text{if}\quad \lVert \nabla f(x_t) \rVert_2 \geq \varepsilon_0
  \qquad\text{(\Lemma{large-mean-bias-control})}
\\ 2 \alpha_t \kappa &\text{if}\quad \lVert \nabla f(x_t) \rVert_2 \leq \varepsilon_0
  \qquad\text{(\Lemma{small-mean-bias-control})}
\end{cases}$$

\begin{proof}
  Let $\beta_t = L_0 + \lVert \nabla f(x_t) \rVert_2 L_1$.
  Using the notation $g = \P (\nabla f_\xi(x_t))$, we have $x_{t+1} = x_t - \eta g$.
  Hence in particular $\lVert x_{t+1} - x_t \rVert_2 = \eta \lVert g \rVert_2 \leq \eta c$,
  and $\eta c \leq \eta \bar\beta / L_1 \leq 1/L_1$,
  thus by definition of the $(L_0, L_1)$-smoothness of $f$, it holds
  $ \Et{f(x_{t+1}) - f(x_t)} \leq - \eta \nabla f(x_t) \cdot \Et{g} + \frac{1}{2} \eta^2 \beta_t \, \Et{\lVert g \rVert_2^2} $.
  The idea is then to use the estimate $\Et{g} \approx \alpha_t \nabla f(x_t)$,
  and absorb as many terms as possible into the large decrease term $- \eta \alpha_t \lVert \nabla f(x_t) \rVert_2^2$.

  Let us show first that $\lVert \nabla f(x_t) \rVert_2 \lVert \alpha_t \nabla f(x_t) - \Et{g} \rVert_2 \leq \frac{1}{2} \alpha_t \lVert \nabla f(x_t) \rVert_2^2 + \kappa$.
  This result is immediate if $\lVert \nabla f(x_t) \rVert_2 \geq \varepsilon_0$ by \Lemma{large-mean-bias-control}.
  Otherwise if $\lVert \nabla f(x_t) \Vert_2 \leq \varepsilon_0$,
  using Young inequality and \Lemma{small-mean-bias-control},
  $$ \lVert \nabla f(x_t) \rVert_2 \lVert \alpha_t \nabla f(x_t) - \Et{g} \rVert_2
  \leq \frac{\alpha_t}{2} \lVert \nabla f(x_t) \rVert_2^2 + \frac{1}{2 \alpha_t} \lVert \alpha_t \nabla f(x_t) - \Et{g} \rVert_2^2
  \leq \frac{1}{2} \alpha_t \lVert \nabla f(x_t) \rVert_2^2 + \kappa
  $$
  Let us show
  $\beta_t \cdot \Et{\lVert g \rVert_2^2} \leq \bar\beta \left( \alpha_t \lVert \nabla f(x_t) \rVert_2^2 + \min(\sigma^2, c^2) \right)$
  to absorb part of the last term.
  If $\lVert \nabla f(x_t) \rVert_2 \leq c$,
  $\Et{ \lVert g \rVert_2^2}
  \leq \Et{ \lVert \nabla f_\xi(x_t) \rVert_2^2}
  = \lVert \nabla f(x_t) \rVert_2^2 + \Et{\lVert \nabla f(x_t) - \nabla f_\xi(x_t) \rVert_2^2}
  \leq \alpha_t \lVert \nabla f(x_t) \rVert_2^2 + \sigma^2$
  and $\beta_t \leq \bar\beta$ concludes.
  Otherwise if $\lVert \nabla f(x_t) \rVert_2 \geq c$,
  then $\beta_t \leq \bar\beta \, \lVert \nabla f(x_t) \rVert_2 / c$
  so $\beta_t \Et{\lVert g \rVert_2^2} \leq \beta_t c^2 \leq \bar\beta c \lVert \nabla f(x_t) \rVert_2 = \bar\beta \alpha_t \lVert \nabla f(x_t) \rVert_2^2$.

  Injecting these two results into the smoothness bound yields the result
  \vspace{-2pt}
  \begin{align*}
    \Et{f(x_{t+1}) - f(x_t)} &\leq - \eta \nabla f(x_t) \cdot \Et{g} + \frac{1}{2} \eta^2 \beta_t \, \Et{\lVert g \rVert_2^2}
    \\ &\leq - \eta \alpha_t \lVert \nabla f(x_t) \rVert_2^2 + \eta \lVert \nabla f(x_t) \rVert_2 \lVert \alpha_t \nabla f(x_t) - \Et{g} \rVert_2 + \frac{1}{2} \eta^2 \beta_t \,\Et{\lVert g \rVert_2^2}
    \\ &\leq - \frac{\eta}{2} \alpha_t \lVert \nabla f(x_t) \rVert_2^2 + \eta \kappa + \frac{1}{2} \eta^2 \beta_t \,\Et{\lVert g \rVert_2^2}
    \\ &\leq - \frac{\eta}{2} \left(1 - \eta \bar\beta\right) \alpha_t \lVert \nabla f(x_t) \rVert_2^2 + \frac{1}{2} \eta^2 \bar\beta \min(\sigma^2, c^2) + \eta \kappa
    &\qedhere
  \end{align*}
\end{proof}

At this stage, we are essentially ready to conclude by a standard telescoping argument.
In terms of precise convergence criterion, our strengthened descent condition permits
an Average-$\mathrm{L}_1$ criterion, which is stronger than the minimal-norm
criterion presented in \citet[Theorem 3.3]{koloskova2023revisiting} in the same setting,
but weaker than the Average-$\mathrm{L}_2$ criterion typical of unclipped SGD
\citep[Theorem 4.8]{bottou2018optimization}, since
\vspace{-2pt}
$$ \left( \min_{t < T} \, \mathbb{E}[ \lVert \nabla f(x_t) \rVert_2 ]\right) \leq
\left(\frac{1}{T} \sum_{t < T} \mathbb{E}[ \lVert \nabla f(x_t) \rVert_2] \right) \leq
\sqrt{\frac{1}{T} \sum_{t < T} \mathbb{E}[ \lVert \nabla f(x_t) \rVert_2^2 ]} $$
Other than this difference, our rate recovers the known rate of unclipped SGD
in the essentially-unclipped regime $c \to +\infty$.
%
This precise%
\footnote{
  For future reference, we conjecture that the constants $2$ in front of $\zeta$
  and the constants of $\min(\sigma^2/c, \sigma)$ are not tight.
}
 quantitative non-asymptotic rate permits selection of hyperparameters
to minimize this upper-bound as a function of known or estimated quantities such as
training time and variance.

\begin{tcolorbox}[boxrule=0pt,standard jigsaw, opacityback=0, frame hidden,sharp corners,enhanced,borderline west={1pt}{0pt}{gray},left=6pt,right=0pt,top=0pt,bottom=0pt,boxsep=\boxsep]
\begin{theorem}[Convergence speed of Clip-SGD]\label{thm:convergence-speed}
  Let $f : \mathbb{R}^d \to \mathbb{R}$ be $(L_0, L_1)$-smooth
  with $(\inf f) \in \mathbb{R}$, and
  let $(x_t)_t$ be Clip-SGD iterates on $f$ with
  clipping constant $c \in \mathbb{R}_+^*$, stepsize $\eta \in \mathbb{R}_+^*$,
  and variance $\sigma^2 \in \mathbb{R}_+$.

  Define $\bar\beta = L_0 + c L_1$. If $\eta \bar\beta < 1$,
  then for $\Delta := \mathbb{E}[f(x_0)] - (\inf f) \in \mathbb{R}_+$
  and $\zeta := 1 / (1 - \eta \bar\beta)$,
  it holds
  $$ \forall T \geq 1, \quad
  \frac{1}{T} \sum_{t < T} \mathbb{E}[\lVert \nabla f(x_t) \rVert_2] \leq
  \sqrt{\frac{2 \zeta \Delta}{\eta T}} + \frac{2 \zeta \Delta}{\eta T c} + 2 \zeta \sqrt{\eta \bar\beta} \sigma + \zeta \cdot 6 \min(\sigma^2 / c, 3 \sigma)$$
\end{theorem}
\end{tcolorbox}

\begin{proof}
  Using \Theorem{descent}, and summing over iterates, we have that
  $$ \frac{1}{T} \sum_{t < T} \mathbb{E}[ f(x_{t+1}) - f(x_t) ]
  \leq - \frac{\eta}{2} (1 - \eta \bar\beta) \frac{1}{T} \sum_{t < T} \mathbb{E}[ \alpha_t \lVert \nabla f(x_t) \rVert_2^2 ]
  + \frac{\eta^2 \bar\beta}{2} \min(\sigma^2, c^2) + \eta \kappa$$
  We will define $m_t := \alpha_t \lVert \nabla f(x_t) \rVert_2^2 = \min(\lVert \nabla f(x_t) \rVert_2^2, c \lVert \nabla f(x_t) \rVert_2)$ and use the more readable form
  $$ \frac{1}{T} \sum_{t < T} \mathbb{E}[m_t] \leq R := \frac{1}{1 - \eta \bar\beta} \left(
  \frac{2 \Delta}{\eta T} + \eta\bar\beta \min(\sigma^2, c^2) + 2 \kappa \right) $$
  By a quick case disjunction, we observe that $\lVert \nabla f(x_t) \rVert_2 \leq \sqrt{m_t} + m_t / c$.
  Hence using the Jensen inequality
  $$\begin{aligned}
    \frac{1}{T} \sum_{t < T} \mathbb{E}[\lVert \nabla f(x_t) \rVert_2]
    \leq \frac{1}{T} \sum_{t < T} \mathbb{E}[\sqrt{m_t}] + \frac{1}{T} \sum_{t < T} \frac{\mathbb{E}[m_t]}{c}
    \leq \sqrt{\frac{1}{T} \sum_{t < T} \mathbb{E}[m_t]} + \frac{1}{c} \frac{1}{T} \sum_{t < T} \mathbb{E}[m_t]
    \leq \sqrt{R} + \frac{R}{c}
  \end{aligned}$$

  It remains to simplify using $\zeta := 1 / (1 - \eta \bar\beta) \in \interval[open]{1}{+\infty}$
  and subadditivity $\sqrt{x + y} \leq \sqrt{x} + \sqrt{y}$ to get
  $$\begin{aligned}
    \sqrt{R}
    \leq \sqrt{\frac{2 \zeta \Delta}{\eta T}} + \sqrt{\zeta \eta \bar\beta} \sigma + \sqrt{2 \zeta \kappa}
    \qquad\text{and}\qquad
    \frac{R}{c} \leq \frac{2 \zeta \Delta}{\eta T c} + \zeta \eta \bar\beta \min\left( \frac{\sigma^2}{c}, c \right) + \frac{2 \zeta \kappa}{c}
  \end{aligned}$$
  The last simplifications use $\min(\sigma^2 / c, c) \leq \sigma$ by case disjunction on $c$,
  then $\sqrt{\zeta} \leq \zeta$
  and $\eta \bar\beta \leq \sqrt{\eta \bar\beta}$,
  and finally $\sqrt{2 \zeta \kappa} + 2 \zeta \kappa / c \leq \zeta \cdot (\sqrt{2 \kappa} + 2 \kappa / c)$,
  followed by
  \Lemma{floor-simplification}: $\sqrt{2 \kappa} + 2 \kappa / c \leq 6 \min(\sigma^2 / c, 3 \sigma)$
\end{proof}

\paragraph{Conclusion.}
We have tightened the analysis of clipped stochastic gradient descent
on $(L_0, L_1)$-smooth functions,
with an improved convergence criterion of the average gradient magnitude
$(\frac{1}{T} \sum_{t < T} \mathbb{E}[\lVert \nabla f(x_t) \rVert_2])$
instead of the minimal magnitude $(\min_{t < T} \, \mathbb{E}[\lVert \nabla f(x_t) \rVert_2])$,
with precise constants for the convergence rates,
and an improved domain of validity $\eta \bar\beta < 1$
(previously $\eta \bar\beta \leq 1/9$) where the maximal stepsize
is governed by $\bar\beta = L_0 + c \, L_1$.
These strengthened results follow from simplifications of the proofs in \citet{koloskova2023revisiting}
rather than new properties, hence we expect that similar improvements will be possible
in results deriving from the original work, such as \citet{lobanov2026avoiding} in overparameterized models,
and maybe for generalized algorithms leveraging similar proof ideas such as
\citet{pethick2025generalized} in non-euclidean geometry.

\subsubsection*{Acknowledgments}

This research was supported in part by the French National Research Agency under the
France 2030 program, PEPR project FOUNDRY reference ANR-23-PEIA-0003.

\bibliography{bibliography}
\bibliographystyle{tmlr}

\newpage
\appendix
\section{Appendix : Proofs omitted from main text}

To simplify notations in the remainder of the appendix, let $X \in \mathbb{R}^d$ be
a random variable with $\mathbb{E}[X] = \mu \in \mathbb{R}^d$,
and assume that $\mathbb{E}[\lVert X - \mu \rVert_2^2] \leq \sigma^2 \in \mathbb{R}_+$.
We write $c \in \mathbb{R}_+^*$ the clipping constant
and $\varepsilon_0 = 4 \min(\sigma, \sigma^2 / c)$.

\vspace{12pt}

\begin{lemma}\label{lem:large-mean-bias-control}
  If $\lVert \mu \rVert_2 \geq \varepsilon_0 := 4 \min(\sigma, \sigma^2 / c)$
  then $\lVert \P \mu - \E[ \P X ] \rVert_2^2 \leq \min(\lVert \mu \rVert_2^2, c^2) / 4$.
\end{lemma}

The difficulty is getting the $(1/4)$ constant, since a gap between two vectors of norm $c$ could be up to $2c$.
\linebreak[3]
The proof relies on three properties of $\ell_2$-norm projection,
controlling the size of the bias depending on $\lVert \mu \rVert_2$.
\begin{itemize}[itemsep=1pt,topsep=1pt]
  \item \Lemma{large-mean-control}: $\lVert \P \mu - \mathbb{E}[ \P X ] \rVert_2^2 \leq \frac{c^3}{2 \lVert \mu \rVert_2} + \frac{c^2 \sigma^2}{\lVert \mu \rVert_2^2}$
    when $\lVert \mu \Vert_2 \geq c$
  \item \Lemma{unconditional-control}: $\lVert \P \mu - \mathbb{E}[ \P X ] \rVert_2^2 \leq \sigma^2 / 4$ unconditionally
  \item \Lemma{small-mean-control}: $\lVert \P \mu - \mathbb{E}[ \P X ] \rVert_2^2 \leq \left(\frac{ \sigma^2}{4 ( c- \lVert \mu \rVert_2)_+}\right)^2$ when $\lVert \mu \rVert_2 < c$.
\end{itemize}

\begin{proof} We proceed by case disjunction on the ordering of $\lVert \mu \Vert_2$ and $c$,
  and then case disjunction on the ordering of $\min(\lVert \mu \rVert_2, c)$ and $\sigma$,
  all four cases leading to the same conclusion.
  \begin{itemize}[topsep=1pt]
    \item Case 1a: If $\lVert \mu \rVert_2 \geq c$ and $c \geq \sigma$,
      \quad use \Lemma{unconditional-control} then $\sigma \leq c$
      $$ \lVert \P \mu - \mathbb{E}[\P X] \rVert_2^2 \leq \sigma^2 / 4 \leq c^2 / 4 = \min(\lVert \mu \rVert_2^2, c^2) / 4 $$

    \item Case 1b: If $\lVert \mu \rVert_2 \geq c$ and $c \leq \sigma$,
      \quad use \Lemma{large-mean-control}
      and $\lVert \mu \rVert_2 \geq \varepsilon_0 = 4 \min(\sigma, \sigma^2 / c) = 4 \sigma$ then $\sigma \geq c$
      $$ \lVert \P \mu - \mathbb{E}[ \P X] \rVert_2^2 \leq \frac{c^3}{2 \lVert \mu \rVert_2} + \frac{c^2 \sigma^2}{\lVert \mu \rVert_2^2}
      \leq  \frac{c^3}{8 \sigma} + \frac{c^2 \sigma^2}{16 \sigma^2} \leq \left(\frac{1}{8} + \frac{1}{16}\right) c^2 \leq \frac{c^2}{4} = \frac{1}{4} \min(\lVert \mu \rVert_2^2, c^2) $$

    \item Case 2a: If $\lVert \mu \rVert_2 \leq c$ and $\lVert \mu \rVert_2 \geq \sigma$,
      \> use \Lemma{unconditional-control}
      $ \lVert \P \mu - \mathbb{E}[\P X ] \rVert_2^2
      \leq \frac{\sigma^2}{4} \leq \frac{\lVert \mu \rVert_2^2}{4} = \frac{1}{4} \min(\lVert \mu \rVert_2^2, c^2)$.

    \item Case 2b: If $\lVert \mu \rVert_2 \leq c$ and $\lVert \mu \rVert_2 \leq \sigma$,
      \, use \Lemma{small-mean-control}
      and observe $\lVert \mu \rVert_2 \geq \varepsilon_0 = 4 \min(\sigma, \sigma^2 / c)$
      thus $\varepsilon_0 = 4 \sigma^2 / c$ and
      $\sigma \geq \lVert \mu \rVert_2 \geq 4 \sigma^2 / c$ thus $4 \sigma \leq c$,
      and in particular $\lVert \mu \rVert_2 \leq \sigma \leq c / 4$ so
      before squaring:
      \begin{align*}
      \lVert \P \mu - \mathbb{E}[ \P X ] \rVert_2 \leq
      \frac{\sigma^2}{4 (c - \lVert \mu \rVert_2)} \leq \frac{\sigma^2}{4 \cdot (c - \frac{1}{4} c)}
      = \frac{1}{3} \frac{\sigma^2}{c} \leq \frac{1}{3} \frac{\lVert \mu \rVert_2}{4}
      \leq \frac{1}{2} \lVert \mu \rVert_2 \leq \frac{1}{2} \min(\lVert \mu \rVert_2, c)
        &\qedhere
      \end{align*}
  \end{itemize}
\end{proof}

\begin{lemma}\label{lem:small-mean-bias-control}
  If $\lVert \mu \rVert_2 \leq \varepsilon_0 := 4 \min(\sigma, \frac{\sigma^2}{c})$
  then $\lVert \P \mu - \mathbb{E}[ \P X ] \rVert_2^2 \leq 2 \min\left(1, \frac{c}{\lVert \mu \rVert_2}\right) \kappa$,
  for $\kappa = \varepsilon_0 \min(\frac{\varepsilon_0}{8}, 2c)$.
\end{lemma}

\begin{proof}
  Define $b = \P \mu - \mathbb{E}[ \P X ]$.
  Let us show first that $\lVert b \rVert_2 \leq \varepsilon_0 / 2$.

  If $c \leq 4 \sigma$, then $\sigma / 2 \leq 2 \sigma^2 / c$,
  thus by \Lemma{unconditional-control}, $\lVert b \rVert_2 \leq \sigma / 2 \leq \min(\sigma / 2, 2 \sigma^2 / c) \leq \min(2 \sigma, 2 \sigma^2 / c) = \varepsilon_0 / 2$.
  \linebreak[3]
  Otherwise it holds $c \geq 4 \sigma$, thus $\varepsilon_0 = 4 \min(\sigma, \sigma^2 / c) = 4 \sigma^2 / c \leq c / 4$.
  In particular $\lVert \mu \rVert_2 \leq \varepsilon_0 \leq c/4$
  implies by \Lemma{small-mean-control} that
  $\lVert b \rVert_2 \leq \sigma^2 / (4 (c - \frac{1}{4} c)) = \sigma^2 / (3c) \leq 2 \sigma^2 / c = \varepsilon_0 / 2$.
  Thus in both cases it holds $\lVert b \rVert_2 \leq \varepsilon_0 / 2$.

  Then, proceed by case disjunction on $\kappa = \varepsilon_0 \, \min(\varepsilon_0 / 8, 2 c)$.
  \begin{itemize}[itemsep=1pt, topsep=1pt]
    \item If $\kappa = \varepsilon_0^2 / 8$,
      \quad observe that $\lVert b \rVert_2 \leq 2 c$, and by \Lemma{unconditional-control}
      $\lVert b \rVert_2 \leq \sigma / 2$, hence using $\min(x,y) \leq \sqrt{xy}$, we get
      $\lVert b \rVert_2^2 \leq \min(\sigma^2 / 4, 4 c^2) \leq \min(\sigma^2, \sigma c) = c \varepsilon_0 / 4$.
      Coupling this with $\lVert b \rVert_2 \leq \varepsilon_0 / 2$ as shown above, we get
      $ \lVert b \rVert_2^2 \leq \min(\varepsilon_0^2 / 4, c \varepsilon_0 / 4)
      = \min(1, c / \varepsilon_0) \, \varepsilon_0^2/ 4 = 2 \min(1, c / \varepsilon_0) \,\kappa
      \leq 2 \min(1, c / \lVert \mu \rVert_2) \,\kappa $.
    \item If $\kappa = 2 \varepsilon_0 c$,
      \quad then use $\lVert b \rVert_2 \leq \varepsilon_0 / 2$ and $\lVert b \rVert_2 \leq 2 c$
      to get
      \begin{align*}
        \frac{1}{2} \frac{\lVert b \rVert_2^2}{\min(1, c / \lVert \mu \rVert_2)}
        = \max\left( \frac{\lVert b \rVert_2^2}{2}, \frac{\lVert b \rVert_2^2 \, \lVert \mu \rVert_2}{2 c} \right)
        \leq \max\left( \frac{1}{2} (\varepsilon_0 / 2) (2 c), \frac{4 c^2 \varepsilon_0}{2 c} \right)
        = 2 \varepsilon_0 c = \kappa
        &\qedhere
      \end{align*}
  \end{itemize}
\end{proof}

\begin{lemma}\label{lem:floor-simplification}
  Let $\kappa = \varepsilon_0 \, \min(\varepsilon_0 / 8, 2 c)$,
  for $\varepsilon_0 = 4 \min(\sigma, \sigma^2 / c)$.
  It holds $ \sqrt{2 \kappa} + 2 \kappa / c \leq 6 \min(\sigma^2 / c, 3 \sigma) $
\end{lemma}

\begin{proof}
  By direct simplification,
  $2 \kappa / c \leq \min(\frac{\varepsilon_0^2}{4c}, 4 \varepsilon_0)$
  and $\sqrt{2 \kappa} \leq \varepsilon_0 / 2$.

  Thus $\sqrt{2 \kappa} + \frac{2 \kappa}{c} \leq \frac{\varepsilon_0}{2} + \frac{\varepsilon_0^2}{4c}
  \leq \frac{1}{2} \frac{4 \sigma^2}{c} + \frac{(4 \sigma)^2}{4 c} = 6 \frac{\sigma^2}{c} $.
  Additionally,
  $\sqrt{2 \kappa} + \frac{2 \kappa}{c}
  \leq \frac{\varepsilon_0}{2} + 4 \varepsilon_0 \leq \frac{9}{2} \varepsilon_0 \leq \frac{9}{2} 4 \sigma = 18 \sigma $
\end{proof}

\newpage

\begin{lemma}\label{lem:large-mean-control}
  If $\lVert \mu \rVert_2 \geq c$
  then $\lVert \P \mu - \mathbb{E}[ \P X ] \rVert_2^2 \leq \dfrac{c^3}{2 \lVert \mu \rVert_2} + \dfrac{c^2 \sigma^2}{\lVert \mu \rVert_2^2}$
\end{lemma}

\begin{proof}
  Define $V := \mathbb{V}[\P X] = \mathbb{E}[ \lVert \P X \rVert_2^2 ] - \lVert \mathbb{E}[ \P X ] \rVert_2^2$.
  Observe that $\lVert \P X \rVert_2 \leq c$ implies $\lVert \mathbb{E}[\P X] \rVert_2^2 \leq c^2 - V$.
  $$\begin{aligned}
    \lVert \P \mu - \mathbb{E} \P X \rVert_2^2 &= \lVert \P \mu \rVert_2^2 - 2 (\P \mu) \cdot \mathbb{E}[ \P X] + \lVert \mathbb{E}[ \P X ] \rVert_2^2
    \leq 2 c^2 - V - 2 (\P \mu) \cdot \mathbb{E}[\P X ]
  \end{aligned}$$
  In order to bound $(\P \mu) \cdot \mathbb{E}[ \P X ]$ below by a correlation argument,
  let us first note that for all $v \in \mathbb{R}^d$ it holds
  $(v \cdot \P v) = \lVert v \Vert_2 \min(\lVert v \rVert_2, c) \geq c \lVert v \rVert_2 - c^2 / 4$.
  This is immediate if $\lVert v \rVert_2 \geq c$ because $v \cdot \P v = c \lVert v \rVert_2$.
  Otherwise $\lVert v \rVert_2 \leq c$ implies
  $\lVert v \rVert_2^2 - c \lVert v \rVert_2 + c^2 / 4 = (\lVert v \rVert_2 - c/2)^2 \geq 0$
  and the result follows by rearranging.

  By careful rearrangement using $\mathbb{E}[ X - \mu ] = 0$,
  then Cauchy-Schwarz, and Jensen $\mathbb{E}[\lVert X \rVert_2] \geq \lVert \mu \rVert_2$,
  we get
  $$\begin{aligned}
    \mu \cdot \mathbb{E}[ \P X]
    &= \mathbb{E}[ \mu \cdot \P X ]
    = \mathbb{E}[ X \cdot \P X] - \mathbb{E}[ (X - \mu) \cdot \P X ]
    \\ &= \mathbb{E}[ X \cdot \P X] - \mathbb{E}[ (X - \mu) \cdot (\P X - \mathbb{E}[ \P X]) ]
    &\geq \left(c \lVert \mu \rVert_2 - c^2 / 4 \right) - \sqrt{\sigma^2 \cdot V}
  \end{aligned}$$
  Moreover, $\lVert \mu \rVert_2 \geq c$ implies that $\P \mu = (c / \lVert \mu \rVert_2) \, \mu$,
  hence reinjecting this into the previous inequality
  $$\begin{aligned}
    \lVert \P \mu - \mathbb{E}[ \P X ] \rVert_2^2 &\leq 2 c^2 - V - 2 (\P \mu) \cdot \mathbb{E}[ \P X]
    \\ &\leq 2 c^2 - V - \frac{2 c}{\lVert \mu \rVert_2} \left( c \lVert \mu \rVert_2 - \frac{c^2}{4} - \sigma \sqrt{V} \right)
    \\ &\leq \frac{c^3}{2 \lVert \mu \rVert_2} + \frac{2 c \sigma}{\lVert \mu \rVert_2} \sqrt{V} - V
    \leq \frac{c^3}{2 \lVert \mu \rVert_2} + \frac{c^2 \sigma^2}{\lVert \mu \rVert_2^2}
  \end{aligned}$$
  where the last step is the maximization
  $\forall s \geq 0, \> a s - s^2 \leq a^2 / 4$ (attained at $s = a /2$),
  used with $s = \sqrt{V}$.
\end{proof}

\begin{lemma}\label{lem:unconditional-control}
  $\lVert \P \mu - \mathbb{E}[ \P X ] \rVert_2 \leq \sigma / 2$.
\end{lemma}

\begin{proof}
  Using firm non-expansiveness of projection
  $(\P \mu - \P X) \cdot (\mu - X) \geq \lVert \P \mu - \P X \rVert_2^2$,
  (see \Lemma{non-expansiveness})
  $$\begin{aligned}
    \left\lVert \P \mu - \P X - \frac{\mu - X}{2} \right\rVert_2^2
    &= \left\lVert \P \mu - \P X  \right\rVert_2^2 - (\P \mu - \P X) \cdot (\mu - X) + \frac{1}{4} \lVert \mu - X \rVert_2^2
    \leq \frac{1}{4} \lVert \mu - X \rVert_2^2
  \end{aligned}$$
  Hence for
  $Z = \P \mu - \P X - \frac{1}{2} (\mu - X)$,
  Jensen inequality gives
  $\lVert \mathbb{E}[Z] \rVert_2^2 \leq \mathbb{E}[ \lVert Z \rVert_2^2 ] \leq \sigma^2 / 4$,
  thus
  \begin{align*}
    \left\lVert \P \mu - \mathbb{E}[\P X] \right\rVert_2
    = \left\lVert \mathbb{E}\left[ \P \mu - \P X - \frac{\mu - X}{2} \right] \right\rVert_2
    \leq \frac{\sigma}{2}
    &\qedhere
  \end{align*}
\end{proof}

\begin{lemma}\label{lem:small-mean-control}
  If $\lVert \mu \rVert_2 < c$
  then $\lVert \P \mu - \mathbb{E}[ \P X ] \rVert_2 \leq \dfrac{\sigma^2}{4 (c - \lVert \mu \rVert_2)}$.
\end{lemma}

\begin{proof}
  Observe that for all $v \in \mathbb{R}^d$, it holds $\lVert v - \P v \rVert_2 = (\lVert v \rVert_2 - c)_+$.

  Additionally, let us show for any $\tau > 0$ that $\forall s \geq 0, \> (s - \tau)_+ \leq s^2 / (4 \tau)$.
  If $s \leq \tau$, the statement is trivial. If $s \geq \tau$, then observe that
  $s^2 - 4 \tau s + 4 \tau^2 = (s - 2 \tau)^2 \geq 0$, thus rearranging
  $s^2 / (4 \tau) \geq s - \tau = (s - \tau)_+$.

  Therefore, when $\lVert \mu \rVert_2 < c$ and thus $\P \mu = \mu$, we get
  with Jensen inequality and the above two bounds that
  \begin{align*}
  \lVert \mu - \mathbb{E} \P X \rVert_2
    &= \lVert \mathbb{E}[ X - \P X ] \rVert_2
    \leq \mathbb{E}[ \lVert X - \P X \rVert_2 ]
    \leq \mathbb{E}[ ( \lVert X \rVert_2 - c )_+ ]
    \\ &
    \leq \mathbb{E}[ ( \lVert X- \mu \rVert_2 - (c - \lVert \mu \rVert_2))_+ ]
    \leq \mathbb{E}\left[ \frac{\lVert X - \mu \rVert_2^2}{4 (c - \lVert \mu \rVert_2)} \right]
    \leq \frac{\sigma^2}{4 (c - \lVert \mu \rVert_2)}
    \qedhere
  \end{align*}
\end{proof}

\begin{lemma}\label{lem:non-expansiveness}
  Projection is firmly non expansive, i.e.
  $ \> \forall (x,z), \>\> (\P x - \P z) \cdot (x - z) \geq \lVert \P x - \P z \rVert_2^2 $
\end{lemma}

\begin{proof}
  $\P x$ minimizes $(y \mapsto \frac{1}{2} \lVert x - y \rVert_2^2)$ over
  $C := \{ y \in \mathbb{R}^d \mid \lVert y \rVert_2 \leq c \}$,
  thus $(x - \P x) \cdot (y - \P x) \leq 0, \forall y \in C$ (normal cone condition).
  Picking $y = \P z$ we obtain the following inequality (and the other by symmetry)
  $$ (x - \P x) \cdot (\P z - \P x) \leq 0 \quad\text{and}\quad
  (z - \P z) \cdot (\P x - \P z) \leq 0$$
  Adding the two, we get
  $ ((x - z) - (\P x - \P z)) \cdot (\P z - \P x) \leq 0 $
  hence
  $ \lVert \P z - \P x \rVert_2^2  \leq (z - x) \cdot (\P z - \P x)$.
\end{proof}

\end{document}

%% file: math_commands.tex
\usepackage{amsmath,amssymb,amsfonts,bm}

\def\eqref#1{equation~\ref{#1}}

\def\1{\bm{1}}

\DeclareMathAlphabet{\mathsfit}{\encodingdefault}{\sfdefault}{m}{sl}
\SetMathAlphabet{\mathsfit}{bold}{\encodingdefault}{\sfdefault}{bx}{n}

\newcommand{\E}{\mathbb{E}}

